\documentclass{article}

\author{
   Satyaki Mukherjee*\\
  \texttt{Technical University of Munich}
  \and
  Soumendu Sundar Mukherjee*\\
  \texttt{Indian Statistical Insititute, Kolkata}
 \and
  Debarghya Ghoshdastidar\\
  \texttt{Technical University of Munich}  
}
\title{Wasserstein Projection Pursuit of Non-Gaussian Signals}
\usepackage{times}

\usepackage[utf8]{inputenc}
\usepackage{rose}

\newcommand{\dist}{\varrho}
\newtheorem*{corollary*}{Corollary}
\newtheorem*{lemma*}{Lemma}
\newtheorem{assumption}{Assumption}
\let\hat\widehat
\let\tilde\widetilde

\begin{document}

\maketitle

\begin{abstract}
    We consider the general dimensionality reduction problem of locating in a high-dimensional data cloud, a $k$-dimensional non-Gaussian subspace of interesting features. We use a projection pursuit approach---we search for mutually orthogonal unit directions which maximise the 2-Wasserstein distance of the empirical distribution of data-projections along these directions from a standard Gaussian. Under a generative model, where there is a underlying (unknown) low-dimensional non-Gaussian subspace, we prove rigorous statistical guarantees on the accuracy of approximating this unknown subspace by the directions found by our projection pursuit approach. Our results operate in the regime where the data dimensionality is comparable to the sample size, and thus supplement the recent literature on the non-feasibility of locating interesting directions via projection pursuit in the complementary regime where the data dimensionality is much larger than the sample size.  
\end{abstract}

\section{Introduction}
\def\thefootnote{*}\footnotetext{Equal contribution.}\def\thefootnote{\arabic{footnote}}
A central question in statistics and machine learning concerns the recovery of useful or interesting features from data. A huge body of literature exists that focuses on such feature extraction tasks. Often the statistician encounters high-dimensional data of which only a relatively low-dimensional subspace is of interest. A family of algorithms, often described by the umbrella term \emph{projection pursuit} \cite{friedman1974projection, huber1985projection}, are particularly well-suited for such tasks. By restricting attention to low-dimensional subspaces, projection pursuit allows the statistician to evade the so-called ``curse-of-dimensionality'', which plagues most classical learning algorithms in high-dimensional settings. Furthermore, projection pursuit helps the statistician to discard noisy and information-poor features. Some prominent members of this family of techniques include Principal Component Analysis (PCA), Independent Component Analysis (ICA), matching pursuit, etc.

Perhaps the simplest projection pursuit algorithm is PCA (see, e.g., \cite{jolliffe2002principal},\cite{jolliffe2016principal}), which  considers the subspace generated by the top $k$ eigenvectors of the sample covariance matrix. In effect, PCA tries to find linear combinations of the original features which explain the most variability. While very useful in its own regard, PCA is limited by the fact that it only considers variances. Thus it works very well when the superfluous features have considerably lower variance than the signal, e.g., in noise reduction problems. 
On the other hand, consider a setup where the interesting components are non-Gaussian, while the rest are Gaussian of comparable variability. The Gaussian components cannot be treated as noise (in the sense of having smaller variance than the signal component) but are simply superfluous or ``uninteresting''. PCA has difficulty separating the interesting non-Gaussian components in such scenarios.

To overcome this limitation, various methods conceptually similar to PCA have been proposed. A broad class of such methods goes by the name of ICA(\cite{lee1998independent}). Broadly speaking, there are two families of ICA algorithms. One class of algorithms focuses more on ensuring that the signal directions are statistically independent, thus they minimize mutual information. The other focuses on finding directions in which the data is the ``least Gaussian'' (i.e. most interesting). In this paper, we are also interested in the latter objective.

Many approaches are possible for finding non-Gaussian directions, depending upon our definition of ``non-Gaussian''. A natural way to quantify non-Gaussianity would be to measure the deviation of some aspect of a probability measure of interest from that of a standard Gaussian. For instance, one could use measures such as kurtosis(\cite{girolami1996negentropy}) or negentropy (\cite{cao2003comparison},\cite{novey2008complex}). Alternatively, one could measure the deviation from a Gaussian using suitable probability metrics such as the Kolmogorov-Smirnov (KS) distance, the Wasserstein distance, etc.

We now state the general projection pursuit approach in the context of the problem of identifying non-Gaussian components with the following simple example. Consider a $p$-dimensional random vector $X$ which satisfies the following: there is a unknown direction $u_*$ such that $u_*^\top X$ is non-Gaussian, while $(I - u_*u_*^\top) X$, the distribution of $X$ in the orthogonal complement of $u_*$, is $(p - 1)$-dimensional standard Gaussian, and, further, the non-Gaussian component $u_*^\top X$ and the Gaussian part $(I - u_*u_*^\top)X$ are statistically independent.  Suppose $\dist(\nu, \nu')$ is some measure of quantifying the distance between two probability measures $\nu$ and $\nu'$. $\dist$ could be a divergence between probability measures (e.g., the Kullback-Liebler divergence) or a proper metric (e.g., KS distance).
Suppose we observe a sample $X_1, \ldots, X_n$ from $\nu$, the distribution of $X$. Our goal is to recover the unknown direction $u_*$. The main idea of projection pursuit is then to find a unit direction $\hat{u}$ such that the empirical distribution of the data projected on $\hat{u}$ (i.e. $\frac{1}{n}\sum_{i = 1}^n \delta_{\hat{u}^\top X_i}$) is the farthest from the standard Gaussian distribution with respect to $\dist$, i.e. 
\[
    \hat{u} = \argmax_{u \,:\, \norm{u} = 1} \dist\bigg(\frac{1}{n}\sum_{i = 1}^n \delta_{u^\top X_i}, \nu_g\bigg),
\]
where $\nu_g$ is the standard Gaussian measure. We will later formalise a version of this for general $k$. Our paper is interested in the case when $\dist$ is the 2-Wasserstein distance between probability measures with finite second moment. Specifically, we analyse the question of whether the recovered directions can be guaranteed (with high probability) to be from the signal space or not.

Some work in this regard has been done in \cite{bickel2018projection} and \cite{montanari2022overparametrized}. \cite{bickel2018projection} use the KS distance for $\dist$. They show that if the data is purely Gaussian (i.e. in a null model with no interesting directions), then two completely different phenomena occur according as whether the data-dimensionality-to-sample-size ratio $\frac{p}{n}$ goes to zero or infinity. In the former regime, all projections are Gaussians (in fact, this is known from the earlier work of \cite{diaconis1984asymptotics}). On the other hand, in the latter regime, given any arbitrary probability distribution $\tilde{\nu}$, with high probability, one can find a (data-dependent) direction along which the data set is distributed as $\tilde{\nu}$. In other words, one can find directions along which the data is as far from Gaussianity (in the KS metric) as one desires. This means that projection-pursuit can spot fake signal amidst complete noise. When applied to actual data in this regime, there is no way of  knowing if the found direction came from some underlying signal space, or if it is a mirage of signal in a Gaussian desert!

\cite{montanari2022overparametrized} prove a similar result for the 2-Wasserstein distance. Both the above papers argue that under the null model of $N(0,\bbI_p)$, when $p/n$ converges to a sufficiently small constant, the empirical distributions of projections of the data points in every direction are close to the standard Gaussian distribution. This obviously begs the question if, under a spiked alternative, one could find directions along which the data projections are non-Gaussian. \cite[Theorem 4.6]{montanari2022overparametrized} study the question of obtaining such a signal direction under a specific model of supervised learning. 

\textbf{Our contributions.} Our analysis is done in the context of an alternative model of unsupervised learning. We suppose that the sample comes from a spiked Gaussian model, i.e. there is a $k$-dimensional subspace in which the distribution is decidedly not Gaussian. We first show that under subgaussian tail assumptions, in every direction, the empirical distribution formed by the data-projections and the true marginal distribution in the same direction are uniformly close. This result is a substantial extension of similar results in \cite{bickel2018projection} and \cite{montanari2022overparametrized} to a more general setting. Further, we also show, using a peculiar property of the $2$-Wasserstein distance, that one can recover an orthonormal set of vectors which form an approximate basis of the signal space. (This can be thought of as an instance of the general strategy of matching pursuit.) In particular, each recovered vector's component in the independent Gaussian space is inversely proportional to the signal-to-noise ratio. Finally, if the signal-to-noise ratio is sufficiently large, then we give a methodology to accurately estimate $k$, i.e. the dimension of the signal space. This allows one to use our sequential procedure even in cases where very little is known about the signal space.

\section{Set-up}
Given a probability measure $\mu$ in $\bbR^p$, and a vector $v \in \bbR^p$, we define the action of $v$ on $\mu$, $v\sharp \mu$ to be the marginal density of $\mu$ in the direction $v$. In particular,
\begin{definition}
    If $X$ is a random variable in $\bbR^p$ from the measure $\mu$ and $v^\top X$ is the dot product of $v$ and $X$, then $v\sharp \mu$ is defined to be the density function of the real-valued random variable $v^\top X$.
\end{definition}

We also need the notion of the $2$-Wasserstein distance $d_{W_2}(\mu, \nu)$ between two probability measures $\mu$ and $\nu$, defined via
\[
    d_{W_2}^2(\mu, \nu) \coloneqq \inf_{\substack{\text{all couplings $\pi$ of } X, Y \\ X \sim \mu, Y \sim \nu}} \bbE_\pi [X - Y]^2.
\]

We now introduce a formal set-up for the non-Gaussian component recovery problem which we will analyse. 

\begin{assumption}\label{ass:distr}
Suppose that we have data $X_1, \ldots, X_n$ i.i.d. from a $\sigma$-subgaussian distribution $\nu$ on $\R^p$,\footnote{A random variable $X \in \bbR^p$ with mean $\mu$ is subgaussian with parameter $\sigma$ or the distribution is in $SG_p(\sigma)$ iff $P\left( \norm{X -\mu} \geq t \right)  \leq Ce^{-\frac{t^2}{2\sigma^2}}.$}  where $\Psi$ has the following structure: Suppose $X \sim \Psi$.

\begin{enumerate} 
    \item $\bbE X = 0; \var(X) = I_p$\footnote{While for the purposes of our proof, we have assumed that the covariance matrix is identity throughout, this assumption is heuristically not much different from working with whitened data.}.
    \item There is a (unknown) $k$-dimensional subspace $U$ such that $\Pi_U X$ has a sufficiently non-Gaussian distribution, and $\Pi_{U^{\perp}}X$ has a close-to-Gaussian distribution, in the sense that there exists constants $\kappa_1, \kappa_2$ with,
    \[
        \inf_{u \in U, \|u\| = 1} d_{W_2}(u \sharp \Psi , \Phi) > \kappa_1 > \kappa_2 >
        \sup_{u \in U^{\perp}, \|u\| = 1}d_{W_2}(u \sharp \Psi, \Phi).
    \]
    \item $\Pi_U X$ and $\Pi_{U^\perp} X$ are independent.
\end{enumerate}
\end{assumption}
We will denote the gaussian subspace, $U^{\perp}$ by $W$.
Given the sample $X_1, \cdots, X_n$, our goal is to recover the space $U$. The quantities $\kappa_1$ and $\kappa_2$, as we will see below, dictate a separation condition necessary to distinguish the non-Gaussian signal components from the Gaussian part.
 
 To further motivate our setup let us quickly look at a simple distribution satisfying the assumptions above :
 \begin{example}
Let the data $X$ be generated from a mixture of gaussians i.e. $\sum_{i=1}^k \frac{N(u_i, \bbI_p)}{k}$, where $u_i$'s are some vectors in $\bbR^p$. Note quickly that if $v$ is any norm $1$ vector orthogonal to all $u_i$, then $v^\top X$ follows $N(0,1)$ and is independent of $u_i^\top X$. Clearly then this is a specific example of our model, with $U = \textit{Span}\{u_1, ..., u_k\}$ and $W = U^{\perp}$. In fact the projection of $\Psi$ to the subspace $W$ is the distribution $N(0,I_{p-k})$.
\end{example}

Before moving onto the technical results we also quickly define the following notations we will be using throughout.
\begin{definition}  
Given a $p$-dimensional distribution $\nu$ and a $1$-dimensional distribution $\mu$
\[
    d(\nu, \mu) \coloneqq \sup_{u : \norm{u} = 1} d_{w_2}(u\sharp\nu, \mu).
\]  
\end{definition}

Note that when $p=1$, this is simply the $2$-Wasserstein distance between the $\nu$ and $\mu$. For larger $p$, when $\mu =\Phi$, our distance $d$ captures how non-gaussian the distribution $\nu$ can become in a particular direction.

\begin{definition}
    Given a $p$-dimensional distribution $\nu$ and a $1$-dimensional distribution $\mu$
\[
    d_{min}(\nu, \mu) \coloneqq \inf_{u : \norm{u} = 1} d_{w_2}(u\sharp\nu, \mu).
\]  
\end{definition}
In essence, when $\mu = \Phi$, $d_{min}$ gives a measure of separation from the $1$-dimensional gaussian.

\begin{definition}
    Let $W$ be some subspace of $\bbR^p$, and let $q_1, ..., q_t$ be some orthonormal basis of $W$. If $Q$ is the matrix whose columns are given by $q_i$ then given a random variable $X \in \bbR^p$, where $X \sim \mu$, we define $\mu_{|Q}$ to be the distribution of $Q^TX$.
\end{definition}
We remark that as the distance $d(\mu,\Phi)$ is rotationally invariant, given a fixed subspace $W$, $d(\mu_{|Q},\Phi)$ is the same regardless of what orthonormal basis one chooses. Thus one can consider the quantity $d(\mu_{|W},\Phi)$ unambiguously.

Finally for the sake of clarity of our conclusion we define a signal to noise ratio for the distribution $\Psi$, $SNR$ as
\[
  SNR = \sqrt{\frac{d(\Psi,\Phi)^2 - d(\Psi_{|W},\Phi)^2}{d(\Psi, \Phi)^2 - d_{min}(\Psi_{|U},\Phi)^2}}.
\]
Note that based on our definition of $\kappa_1$ and $\kappa_2$, our SNR is always larger than $\sqrt{\frac{d(\Psi,\Phi)^2 - \kappa_2^2}{d(\Psi, \Phi)^2 - \kappa_1^2}}$.
We will show that we can construct with high probability $k$ orthonormal vectors  $v_1, ..., v_k$ such that 
\[
    \norm{\textit{Proj}_{W}(\hat{v}_j)} \leq \frac{2}{SNR}.
\]
We also show that we can estimate $k$ if $k \leq \frac{SNR^2}{4}$.

\section{Main results}

The following proposition is the central pivot granting us leverage to most of our results.

\begin{proposition} \label{thm:pconc-const} 
Let $X_1, ..., X_n$ be $n$ data points from $\Psi$. Let $n$,$p$ go to infinity in a way such that $p/n \rightarrow \gamma$. Then given a positive constant $\epsilon$, there exists a positive constant $\gamma_{\sigma,\epsilon}$ depending on $\sigma$ and $\epsilon$ such that when $\gamma \leq \gamma_{\sigma,\epsilon}$, we have
$$P\left(sup_{u \in S_{p-1}} \abs{d_{W_2}\left(\frac{\Sigma_{i=1}^n \delta_{u^\top X_i}}{n}, u\sharp\Psi\right) - \bbE_{X_1, ..., X_n}\left[d_{W_2}\left(\frac{\Sigma_{i=1}^n \delta_{u^\top X_i}}{n}, u\sharp\Psi\right)\right]}  > \epsilon \right) < De^{-nc_{\sigma, \gamma,\epsilon}},$$
\end{proposition}
where $c_{\sigma,\gamma,\epsilon}$ is some positive constant dependent on $\sigma$, $\gamma$, and $\epsilon$.

Proposition~\ref{thm:pconc-const} uniformly bounds the difference between the data dependent (and thus random) quantity, $d_{W_2}\left(\frac{\Sigma_{i=1}^n \delta_{u^\top X_i}}{n}, u\sharp\Psi\right),$ and the deterministic quantity,$ \bbE_{X_1, ..., X_n}\left[d_{W_2}\left(\frac{\Sigma_{i=1}^n \delta_{u^\top X_i}}{n}, u\sharp\Psi\right)\right],$ dependent only on $u$.

Then given our assumptions~\ref{ass:distr} on the distribution $\Psi$ above, we will state the following theorem (proved in Section~\ref{proof}): 
\begin{theorem}[Empirical non-gaussianity implies true non-gaussianity] \label{thm:nongaus}
        Let $X_1, ..., X_n$ be $n$ data points from $\Psi$. Let $n$,$p$ go to infinity in a way such that $p/n \to \gamma$. Given an $\epsilon > 0$, there exists a constants $\gamma_{\sigma,\epsilon}$ dependent on $\sigma$ and $\epsilon$  and  $C_{\sigma}$  depending on $\sigma$ such that if $\gamma \leq \gamma_{\sigma,\epsilon}$, the following statement is true with high probability for all  unit vectors $u$ in $\bbR^p$ simultaneously : 
        \[\abs{d_{W_2}\left(\frac{\Sigma_{i=1}^n \delta_{u^\top X_i}}{n}, \Phi\right) - d_{W_2}\left(u\sharp\Psi, \Phi\right)}  \leq \epsilon + \frac{C_{\sigma}}{\sqrt[4]{n}}.\]
\end{theorem}

We note that Theorem~\ref{thm:nongaus} needs very little assumptions on the distribution $\Psi$. We only need $\Psi$ to be $\sigma$-subgaussian. The main upshot of the theorem is that it implies with uniform high probability that in every direction the empirical distribution of the projection is as far away from Gaussian, as the true marginal distribution in that direction. Thus heuristically if we want to find directions in which $\Psi$ is not gaussian it makes sense to maximise the quantity $d_{W_2}\left(\frac{\Sigma_{i=1}^n \delta_{u^\top X_i}}{n}, \Phi\right).$ We can now proceed to state conditions under which the recovered directions have a very small component in the Gaussian subspace, $W$. A proof of Theorem~\ref{thm:altmain}  is written in Section~\ref{proof}.

\begin{theorem}[recovered direction is almost orthogonal to gaussian subspace] \label{thm:altmain} Let $W$ be the gaussian subspace of $\Psi$. Let $X_1, ..., X_n$ be $n$ data points from $\Psi$. Let $n$,$p$ go to infinity in a way such that $p/n \to \gamma$. Given $\epsilon > 0$, there exists a constant $\gamma_{\sigma,\epsilon}$ dependent on $\sigma$ and $\epsilon$ such that if $\gamma \leq \gamma_{\sigma,\epsilon}$, then with asymptotic high probability for any $u \in S_{p-1}$ such that $d_{W_2}\left(\frac{\Sigma_{i=1}^n \delta_{u^\top X_i}}{n}, \Phi\right) \geq \sqrt{1-\delta^2}d(\Psi,\Phi) + \epsilon + \frac{C_{\sigma}}{\sqrt[4]{n}}, $
 we have that \[\norm{\textit{Proj}_{W}(u)} \leq \delta\frac{d(\Psi,\Phi)}{\sqrt{d(\Psi, \Phi)^2 - d(\Psi_{|W}, \Phi)^2}}.\]
 \end{theorem} 

Now that we have introduced most of our bulky technology, we can use it to prove the following simple Corollary. In the interest of space we have moved a detailed proof of the Corollary to the Appendix in Section~\ref{app:cor_k-col}. This in turn allows us to argue the validity of procedure in the vein of the general idea of matching pursuit.

\begin{corollary}[Guarantee that recovery is possible]\label{thm:k-col}
    Let $U$ be a $k$-dimensional sub-space of $\bbR^p$, where $k$ is a constant. Let $X_1, ..., X_n$ be $n$ data points from $\Psi$. Let $l < k$ be some integer. Let $v_1 , ..., v_l$ be some vectors in $\bbR^p$. Then there exists some constant $C_{\sigma}$, depending on $\sigma$ such that given $\epsilon > 0$ there exists with high probability a unit vector, $u$ which is orthonormal to all $v_i$ such that 
    \[
        d_{W_2}\bigg(\frac{1}{n}\sum_{i=1}^n \delta_{u^\top X_i}, \Phi\bigg) \geq d_{min}(\Psi_{|U},\Phi) - \epsilon -  \frac{C_{\sigma}}{\sqrt[4]{n}}.
    \]
\end{corollary}

Suppose now that the distribution $\Psi$ is such that there is a $k$ dimensional subspace $U$ such that $d_{min}(\Psi_{|U},\Phi) = \sqrt{1-\delta^2} d(\Psi,\Phi)$. That is the ``top $k$ directions'' are a constant factor far from gaussian as the maximum possible. Then using Corollary~\ref{thm:k-col} we can with high probability sequentially construct vectors $v_1, ..., v_k$ such that for every $1 \leq j \leq k$,
    \[
    d_{W_2}\left(\frac{\Sigma_{i=1}^n \delta_{v_j^TX_i}}{n}, \Phi\right) \geq \sqrt{1-\delta^2} d(\Psi,\Phi) - \epsilon - \frac{C_{\sigma}}{\sqrt[4]{n}}.
    \]
Then setting $\epsilon = \frac{\delta^2 d(\Psi,\Phi)}{2}$, and $n$ large enough such that $\frac{4C_{\sigma}}{\sqrt[4]{n}} \leq  \delta^2 d(\Psi,\Phi)$, we have that 
\[
    d_{W_2}\left(\frac{\Sigma_{i=1}^n \delta_{v_j^TX_i}}{n}, \Phi\right) \geq \sqrt{1-\delta^2} d(\Psi,\Phi) - \epsilon - \frac{C_{\sigma}}{\sqrt[4]{n}} \geq \sqrt{1-4\delta^2} d(\Psi,\Phi) + \epsilon  +\frac{C_{\sigma}}{\sqrt[4]{n}}..
\]
Now we can use Theorem~\ref{thm:altmain} with $\epsilon$ as set above. Thus if $p/n$ converges to a sufficiently small constant $\gamma$, then for large enough $n$ with high probability we have that for every $j$,
\[
\norm{\textit{Proj}_{W}(v_j)} \leq 2\delta\frac{d(\Psi,\Phi)}{\sqrt{d(\Psi, \Phi)^2 - d(\Psi_{|W}, \Phi)^2}}.
\]

In other words the $k$-space that we found (i.e. the one spanned by $v_1, ..., v_k$) is mostly orthogonal to $W$, the subspace where the distribution is close to Gaussian. The above discussion then gives the following natural method to estimate $k$ vectors which are almost orthogonal to $W$. For $1 \leq j \leq k$, let 
\[
    \hat{v}_j =  \argmax_{\norm{\hat{v}_j} = 1; \forall t < j : \hat{v}_j^T\hat{v}_t = 0} d_{W_2}\bigg(\frac{1}{n}\sum_{i=1}^n \delta_{\hat{v}_j^TX_i},\Phi\bigg).
\]

The above discussion then implies that, with high probability for large enough $n$ we have that (by invoking Theorem~\ref{thm:altmain} with $\delta$ such that $d_{min}(\Psi_{|U},\Phi) = \sqrt{1-\delta^2} d(\Psi,\Phi)$):
\[
\norm{\textit{Proj}_{W}(\hat{v}_j)} \leq 2\sqrt{\frac{d(\Psi,\Phi)^2 - d_{min}(\Psi_{|U},\Phi)^2}{d(\Psi, \Phi)^2 - d(\Psi_{|W},\Phi)^2}} = \frac{2}{SNR}.
\]

A common problem that often occurs in such problems is that $k$ is unknown. To give some answer to this question we first consider the following corollary which is proved in detail in the Appendix in Section~\ref{app:cor_choosing-k}.
\begin{corollary}\label{thm:choosing-k}
    Given  integers $m > k+1$, let $\delta$ be a positive real number such that $4\delta^2 < \frac{1}{m} \left(1 - \frac{d(\Psi_{|W}, \Phi)^2}{d(\Psi, \Phi)^2}\right).$ Let $X_1, ..., X_n$ be $n$ data points from $\Psi$. Let $n$,$p$ go to infinity in a way such that $p/n \rightarrow \gamma$ Given $\epsilon > 0$ there is a $\gamma_{\sigma,\epsilon}$, where $\gamma_{\sigma,\epsilon}$ is a constant depending on $\sigma$, $\epsilon$, such that if $\gamma \leq \gamma_{\sigma,\epsilon}$ then with high probability there \textbf{does not exist} a set of $k+1$ orthonormal unit vectors $v_1, ..., v_{k+1}$ such that 
    \[ 
        d_{W_2}\bigg(\frac{1}{n}\sum_{i=1}^n \delta_{v_j^TX_i}, \Phi\bigg) \geq \sqrt{1-4\delta^2} d(\Psi,\Phi) + \epsilon + \frac{C_{\sigma}}{\sqrt[4]{n}}. \]
\end{corollary}

Continuing the discussion prior to the corollary, we consider $\delta$ such that $d_{min}(\Psi_{|U},\Phi) = \sqrt{1-\delta^2} d(\Psi,\Phi)$. Note that when $k+1 < \frac{1}{4}\frac{d(\Psi,\Phi)^2 - d(\Psi_{|W},\Phi)^2}{d(\Psi, \Phi)^2 - d_{min}(\Psi_{|U},\Phi)^2} = \frac{SNR^2}{4}$ the hypothesis of Corollary~\ref{thm:choosing-k} is true. This gives us a natural cutoff point for our sequential algorithm. We can stop at $\hat{k}$, if for $\epsilon = \frac{\delta^2 d(\Psi,\Phi)}{2}$, and $n$ large enough such that $\frac{4C_{\sigma}}{\sqrt[4]{n}} \leq  \delta^2 d(\Psi,\Phi)$ we have that 
    \[
        d_{W_2}\bigg(\frac{1}{n}\sum_{i=1}^n \delta_{\hat{v}_{\hat{k}+1}^TX_i}, \Phi\bigg) < \sqrt{1-4\delta^2} d(\Psi,\Phi) + \epsilon + \frac{C_{\sigma}}{\sqrt[4]{n}}. \]
Corollary~\ref{thm:choosing-k} then implies that this stopping rule ensures with high probability that $\hat{k} \leq k$. On the other hand, the discussion following Corollary~\ref{thm:k-col} means that the same stopping rule ensures that $\hat{k} \geq k$. In effect we have that if $k+1 < \frac{SNR^2}{4}$, then with high probability $\hat{k} = k$.

\section{Conclusion}
In this article, we have considered the problem of isolating a non-Gaussian independent component from a Gaussian counterpart under certain separability assumptions. We have theoretically analysed the approximation accuracy of a projection pursuit procedure. In contrast to more traditional procedures like PCA, we do not need the variances of the superfluous feature directions to be small. We only need a distributional gap between directions which are Gaussian and those which are not. 

Since the proposed method involves optimisation of the objective function $d_{W_2}\big(\frac{1}{n}\sum_{i=1}^n \delta_{v^\top X_i}, \Phi\big)$ as $v$ varies over the unit sphere, two natural questions immediately come to mind. First of all, since our objective function is markedly non-convex, designing an efficient algorithm that can find a global minimum (or even good local minima) would be a significant addition to present work. 

Secondly, it needs to be investigated if similar results are true for distances other than the $2$-Wasserstein distance. It is plausible that some distances would be more suitable both from a theoretical perspective and also the practical optimisation aspect. We leave the investigation of these questions for future work.

\section{Proofs} \label{proof}
\subsection{Proof of Proposition~\ref{thm:pconc-const}}
Quickly noting that for any vector $u \in \bbR^p$, with $\norm{u}=1$ we have, \[ P\left( \abs{u.(X-\mu)} \geq t \right) \leq P\left( \norm{X -\mu} \geq t \right),\]
we get the following simple proposition :
\begin{proposition}
    If $X \in \bbR^p$ is in $SG_p(\sigma)$ and $u \in \bbR^p$ be any unit norm vector (i.e. $\norm{u} = 1$), then $u^\top X \in SG_1(\sigma)$.
\end{proposition}

The following is a simple proposition which bounds the norms of the sample covariance matrix of subgaussian random variables. It is a slightly reworded version of Theorem 6.5 of \cite{wainwright2019high}
\begin{proposition}
\label{lem:norm}
Let $X_1, ..., X_n$ be iid sample from a $\sigma$ subgaussian distribution in $\bbR^p$ with covariance matrix $\bbI$. Then there exists universal constants $c_1,c_2,c_3$ such that we have for all $\delta > 0$ that 
\[P\left( \norm{\frac{\sum X_iX_i^T}{n}}_2 \geq 1 + \sigma^2\left(c_1 \left(\sqrt{\frac{p}{n}} + \frac{p}{n} \right)  + \delta\right) \right) \leq c_2e^{-nc_3\min(\delta, \delta^2)}
\]    
\end{proposition} 

We will also use the following result (for $p=2$) on Wasserstein distances and sample convergences found in \cite{bobkov2019one} as Corollary 7.17. 
\begin{proposition} \label{lem:sampleconverge}
    Let $p$ be some positive integer. Let $\mu$ be some distribution such that for some $s > p$ its $s$'th moment exists and is bounded. Then if $X_1, ..., X_n$ are iid random variables sampled from $\mu$, we have
     \[
        \bbE_{X_1, ..., X_n}\left[d_{W_p}\left(\frac{\Sigma_{i=1}^n \delta_{u^\top X_i}}{n}, \mu \right)^p\right] \leq \frac{C}{\sqrt{n}},
     \]
     where $C$ is some absolute constant dependent upon the upper bound of the $s$'th moment.
\end{proposition}

To prove Proposition~\ref{thm:pconc-const}, we will finally be needing the following lemma (proved in the Appendix in section~\ref{app:w_2-conc},) on the concentration of the 2-Wasserstein distance between a subgaussian measure $\mu$ and the empirical measure $\mu_n = \frac{1}{n}\sum_{i = 1}^n \delta_{X_i}$ of an i.i.d. sample $X_1, \ldots, X_n$ from $\mu$ (a similar result with a log-Sobolev assumption on $\mu$ appears as Theorem~7.1 in \cite{bobkov2019one}).
\begin{lemma}\label{lem:w_2-conc}
Let $\mu$ be a $\sigma$-subgaussian  measure. Let $\mu_n$ be the empirical measure formed from an i.i.d. sample of size $n$ from $\mu$. Then \[
    \bbP\big(\abs{d_{W_2}(\mu_n, \mu) - \bbE d_{W_2}(\mu_n, \mu)} \geq t \big) \leq Ce^{-\frac{cnt^2}{\sigma^2}}.
\]
for some absolute constants $C, c > 0$.
\end{lemma}
Now armed with the above preliminaries we can move onto proving our central results:
\begin{proof} [Proof of Proposition~\ref{thm:pconc-const}] As mentioned before we begin by using Lemma~\ref{lem:w_2-conc} and the hypothesis that $u^\top X$ is $\sigma$-subgaussian to get that for any fixed $u$ (with $\norm{u}=1$) we have\[P\left(\abs{d_{W_2}\left(\frac{\Sigma_{i=1}^n \delta_{u^\top X_i}}{n}, u\sharp\Psi \right) - \bbE_{X_1, ..., X_n} \left[d_{W_2}\left(\frac{\Sigma_{i=1}^n \delta_{u^\top X_i}}{n}, u\sharp\Psi \right)\right]} \geq t\right) \leq A'e^{-\frac{cn\epsilon^2}{\sigma^2}}, \]
where the constants $A',c$ are absolute constants

Let $E_{p,\delta}$ be the smallest delta net on the unit sphere in $\bbR^p$, i.e. given any $u \in \bbS_{p-1}$, $\exists$ a $v \in E_{p,\delta}$ such that $\norm{u-v} \leq \delta$. It is known that there exists such a net for any $p$ such that $\abs{E_{p,\delta}} \leq A''\frac{1}{\delta^p}$. Then we have that \[P\left(\underset{u \in E_{p,\delta}}{\sup} \abs{d_{W_2}\left(\frac{\Sigma_{i=1}^n \delta_{u^\top X_i}}{n}, u\sharp\Psi \right) - \bbE_{X_1, ..., X_n} \left[d_{W_2}\left(\frac{\Sigma_{i=1}^n \delta_{u^\top X_i}}{n}, u\sharp\Psi \right)\right]} \geq \epsilon\right) \leq \left(\frac{1}{\delta}\right)^pAe^{-\frac{cn\epsilon^2}{\sigma^2}}.\]

To go from taking the supremum over the net to that on the entire sphere then we would have to control how small changes in $u$ affect the quantity of interest.  To that end let $u \in \bbS_{p-1}$. Let $v \in \bbE_{p,\delta}$ such that $\norm{u-v} \leq \delta$. Then we have 
\begin{align*}
    \abs{d_{W_2}\left(\frac{\Sigma_{i=1}^n \delta_{u^\top X_i}}{n}, u\sharp\Psi \right) - d_{W_2}\left(\frac{\Sigma_{i=1}^n \delta_{v^\top X_i}}{n}, v.\Psi \right)} &\leq \abs{d_{W_2}\left(\frac{\Sigma_{i=1}^n \delta_{u^\top X_i}}{n}, u\sharp\Psi \right) - d_{W_2}\left(\frac{\Sigma_{i=1}^n \delta_{v^\top X_i}}{n}, u\sharp\Psi \right)} \\ &+ \abs{d_{W_2}\left(\frac{\Sigma_{i=1}^n \delta_{v^\top X_i}}{n}, u\sharp\Psi \right) - d_{W_2}\left(\frac{\Sigma_{i=1}^n \delta_{v^\top X_i}}{n}, v.\Psi \right)} \\
    &\leq \abs{d_{W_2}\left(\frac{\Sigma_{i=1}^n \delta_{u^\top X_i}}{n}, \frac{\Sigma_{i=1}^n \delta_{v^\top X_i}}{n} \right)} + \abs{d_{W_2}\left(u\sharp\Psi, v.\Psi \right)}
\end{align*}

Let us then quickly bound both of the two terms above. The second term can be upper bounded as 
\begin{align*}
    \abs{d_{W_2}\left(u\sharp\Psi, v.\Psi\right)}^2 &= \inf_{\substack{(X,X') : \textit{ the marginals } \\ X \textit{ and } X' \textit{ distributed as } \Psi}} \bbE[(u^\top X - v^\top X')^2] \\
    &\leq \bbE_{X \sim \Psi}[(u^\top X - v^\top X)^2] & (\textit{Considering the coupling } X = X') \\
    &= (u-v)^T \bbE[XX^T] (u-v) \\
    &\leq \delta^2. &(\textit{as } \bbE[XX^T] = \bbI_p)
\end{align*}

Similarly for the first term we note 
\begin{align*}
    \abs{d_{W_2}\left(\frac{\Sigma_{i=1}^n \delta_{u^\top X_i}}{n}, \frac{\Sigma_{i=1}^n \delta_{v^\top X_i}}{n} \right)}^2 &\leq \sum_{i=1}^n \frac{(u^\top X_i - v^\top X_i)^2}{n} \\
    &= (u-v)^T \left(\frac{1}{n}\sum_{i=1}^n X_iX_i^T \right) (u-v) \\
    &\leq \left(\delta \norm{\hat{\Sigma}_n}_2\right)^2,
\end{align*}
where $\hat{\Sigma}_n$ is the sample covariance matrix and $\norm{.}_2$ denotes the operator or $L_2$ norm. We will now use a technical claim that the operator norm of the sample covariance matrix is with high probability smaller than $2 + C'\sigma^2\left(\sqrt{\frac{p}{n}} +\frac{p}{n}\right)$, for some universal constant $C'$. This follows from invoking Proposition~\ref{lem:norm}.

Then we can condition on this event as this is true with high probability ($1-e^{-\theta n}$). Thus we have whp 
\[\abs{d_{W_2}\left(\frac{\Sigma_{i=1}^n \delta_{u^\top X_i}}{n}, \frac{\Sigma_{i=1}^n \delta_{v^\top X_i}}{n} \right)} \leq \delta \left(3 + C'\sigma^2\left(\sqrt{\frac{p}{n}}+\frac{p}{n}\right)\right).\] 
Combining everything then we have whp \[\abs{d_{W_2}\left(\frac{\Sigma_{i=1}^n \delta_{u^\top X_i}}{n}, u\sharp\Psi \right) - d_{W_2}\left(\frac{\Sigma_{i=1}^n \delta_{v^\top X_i}}{n}, v.\Psi \right)} \leq \delta \left(3 + C'\sigma^2\left(\sqrt{\frac{p}{n}}+\frac{p}{n}\right)\right).\]

We then have that whenever there exists an $u \in \bbS_{p-1}$ with 
\[
\abs{d_{W_2}\left(\frac{\Sigma_{i=1}^n \delta_{u^\top X_i}}{n}, u\sharp\Psi \right) - \bbE_{X_1, ..., X_n} \left[d_{W_2}\left(\frac{\Sigma_{i=1}^n \delta_{u^\top X_i}}{n}, u\sharp\Psi \right)\right]} \geq \epsilon + 2\delta \left(3 + C'\sigma^2\left(\sqrt{\frac{p}{n}}+\frac{p}{n}\right)\right),
\] 
there exists a $v \in E_{p,\delta}$ whp (with the property that $\norm{u-v}\leq \delta$) such that $$\abs{d_{W_2}\left(\frac{\Sigma_{i=1}^n \delta_{v^\top X_i}}{n}, u\sharp\Psi \right) - \bbE_{X_1, ..., X_n} \left[d_{W_2}\left(\frac{\Sigma_{i=1}^n \delta_{u^\top X_i}}{n}, u\sharp\Psi \right)\right]} \geq \epsilon.$$

Thus using the probability bound on the delta net gives us 
    \begin{align*}&P\left(\underset{u \in \bbS_{p-1}}{\sup} \abs{d_{W_2}\left(\frac{\Sigma_{i=1}^n \delta_{u^\top X_i}}{n}, u\sharp\Psi \right) - \bbE_{X_1, ..., X_n} \left[d_{W_2}\left(\frac{\Sigma_{i=1}^n \delta_{u^\top X_i}}{n}, u\sharp\Psi \right)\right]} \geq \epsilon + 2\delta  \left(3 + C'\sigma^2\sqrt{\frac{p}{n}}+C'\sigma^2\frac{p}{n} \right)\right) \\ 
    &\leq P\left(\underset{v \in E_{p,\delta}}{\sup} \abs{d_{W_2}\left(\frac{\Sigma_{i=1}^n \delta_{v^\top X_i}}{n}, v.\Psi \right) - \bbE_{X_1, ..., X_n} \left[d_{W_2}\left(\frac{\Sigma_{i=1}^n \delta_{v^\top X_i}}{n}, v.\Psi \right)\right]} \geq \epsilon\right) + e^{-\theta n}\\
    &\leq  \left(\frac{1}{\delta}\right)^pAe^{-\frac{cn\epsilon^2}{\sigma^2}} + e^{-\theta n} = Ae^{-\frac{cn\epsilon^2}{\sigma^2} - p\log \delta} + e^{-\theta n},
    \end{align*}
    where we get the extra $\frac{p}{n}$ term as it is no longer true that $\frac{p}{n} \ll \sqrt{\frac{p}{n}}$ (and from invoking Proposition~\ref{lem:norm})
    Using the hypothesis that $\frac{p}{n} \rightarrow \gamma$ and choosing a $\delta$ such that \[\delta = \frac{\epsilon}{2\left(3 + C'\sigma^2\sqrt{\gamma} + C'\sigma^2\gamma\right)},\]
    we get :
    \begin{align*}
    &P\left(\underset{u \in \bbS_{p-1}}{\sup} \abs{d_{W_2}\left(\frac{\Sigma_{i=1}^n \delta_{u^\top X_i}}{n}, u\sharp\Psi \right) - \bbE_{X_1, ..., X_n} \left[d_{W_2}\left(\frac{\Sigma_{i=1}^n \delta_{u^\top X_i}}{n}, u\sharp\Psi \right)\right]} \geq 2\epsilon\right) \\
    &\leq A\exp\left(-\frac{cn\epsilon^2}{\sigma^2} - p \log \frac{\epsilon}{6 + 2C'\sigma^2\sqrt{\gamma} + 2C'\sigma^2\gamma }\right) + e^{-\theta n} \\
    &\rightarrow A\exp -n\left(\frac{c\epsilon^2}{\sigma^2} + \gamma \log \frac{\epsilon}{6 + 2C'\sigma^2\sqrt{\gamma} + 2C'\sigma^2\gamma }\right) + e^{-\theta n}
    \end{align*}

    As there is some constant $\gamma_{\sigma,t}$ such that for all $\gamma \leq \gamma_{\sigma,\epsilon}$, $\frac{c\epsilon^2}{\sigma^2} + \gamma \log \frac{\epsilon}{6 + 2C'\sigma^2\sqrt{\gamma} + 2C'\sigma^2\gamma }$ is positive and lower bounded, the proof follows.
\end{proof}

\subsection{Proof of Theorem~\ref{thm:nongaus}}
\begin{proof}[Proof of Theorem~\ref{thm:nongaus}]         
    Note that Proposition~\ref{thm:pconc-const} probabilistically bounds the difference between $d_{W_2}\left(\frac{\Sigma_{i=1}^n \delta_{u^\top X_i}}{n}, u\sharp\Psi\right)$ and the deterministic quantity $\bbE_{X_1, ..., X_n}\left[d_{W_2}\left(\frac{\Sigma_{i=1}^n \delta_{u^\top X_i}}{n}, u\sharp\Psi\right)\right]$ as $u$ varies over all possible unit vectors. Define $\gamma_{\sigma, \epsilon}$ as in Proposition~\ref{thm:pconc-const}. Thus when $\gamma \leq \gamma_{\sigma, \epsilon}$ we have with high probability for all unit vectors $u \in \bbR^p$ simultaneously that  
    \begin{align} \label{eq:pconc-1}
    \abs{d_{W_2}\left(\frac{\Sigma_{i=1}^n \delta_{u^\top X_i}}{n}, u\sharp\Psi\right) - \bbE_{X_1, ..., X_n}\left[d_{W_2}\left(\frac{\Sigma_{i=1}^n \delta_{u^\top X_i}}{n}, u\sharp\Psi\right)\right]}  \leq \epsilon           
    \end{align} 

    Since $u\sharp\Psi$ is $\sigma$-subgaussian we can conclude that all its moments are upper bounded (by a suitable function of $\sigma$). Choose any $s > 4$ and combine the upper bound on $s$'th moment of $u^\top X$ with Proposition~\ref{lem:sampleconverge} (Corollary 7.17 of \cite{bobkov2019one}).  We then get that there is a constant $C_{\sigma}$, dependent on $\sigma$, such that  
    \begin{align}\label{eq:pconc-2}
    \bbE\left[d_{W_2}\left(\frac{\Sigma_{i=1}^n \delta_{u^\top X_i}}{n}, u\sharp\Psi \right)\right]^2 \leq \bbE\left[d_{W_2}\left(\frac{\Sigma_{i=1}^n \delta_{u^\top X_i}}{n}, u\sharp\Psi \right)^2\right] \leq \frac{C_{\sigma}^2}{\sqrt{n}}.        
    \end{align}

       Combining equations \ref{eq:pconc-1} and \ref{eq:pconc-2} with the triangle inequality we can write with high probability that
    \begin{align*} 
        \abs{d_{W_2}\left(\frac{\Sigma_{i=1}^n \delta_{u^\top X_i}}{n}, \Phi\right) - d_{W_2}\left(u\sharp\Psi, \Phi\right)} &\leq d_{W_2}\left(\frac{\Sigma_{i=1}^n \delta_{u^\top X_i}}{n}, u\sharp\Psi\right) \\
        &\leq \epsilon + \bbE\left[d_{W_2}\left(\frac{\Sigma_{i=1}^n \delta_{u^\top X_i}}{n}, u\sharp\Psi\right)\right] \leq \epsilon + \frac{C_{\sigma}^2}{\sqrt{n}}.
    \end{align*}
\end{proof}

\subsection{Proof of Theorem~\ref{thm:altmain}}
\begin{proof}[Proof of Theorem~\ref{thm:altmain}] 
Let $u = \alpha v + \sqrt{1-\alpha^2} w$, where $w \in W$, $v \in W^{\perp}$ and $\norm{w} = \norm{v} = 1$. Let $X$ be a random variable distributed as $\Psi$. 
Thus $u^\top X = \alpha v^\top X + \sqrt{1 - \alpha^2} w^\top X$, where $v^\top X$ and $w^\top X$ are independent random variables following the distributions $v.\Psi$ and $w\sharp\Psi$ respectively. In other words if $Y_1$ and $Y_2$ are two independent random variables from the distributions $v.\Psi$ and $w\sharp\Psi$ respectively we can write \begin{align} \label{eq:couplingXYZ}
    u^\top X \eqd \alpha Y_1 + \sqrt{1-\alpha^2} Y_2,
\end{align} where $\eqd$ means equal in distribution. Note then that if $Z$ is a random variable distributed as $\Phi$(i.e.  $N(0,1)$), and $Z_2, Z_3$ are two iid copies distributed as $\Phi$, we can also write \begin{align} \label{eq:couplingZZZ}
  Z \eqd \alpha Z_1 + \sqrt{1-\alpha^2} Z_2.  
\end{align}
Let $\Omega$,  be the set of all possible couplings of the distributions of $u\sharp\Psi$ and $\Phi$. Similarly, let $\Omega_1$ (resp. $\Omega_2$) be the set of all possible couplings of the distributions of $v.\Psi$ and $\Phi$ (resp. $w\sharp\Psi$ and $\Phi$). Then there is a natural way to construct a coupling in $\Omega$ given a coupling in $\Omega_1$ and another in $\Omega_2$. That given any joint distribution $\mu$ in $\Omega_1$ whose marginals are $v.\Psi$ and $\Phi$ respectively, define $Y_1$ and $Z_1$ be the corresponding marginal random variables i.e. $(Y_1,Z_1) \sim \mu$. Similarly given any joint distribution $\nu$ in $\Omega_2$, we can define random variables $Y_2, Z_2$ where $(Y_2,Z_2) \sim \nu.$ Note that by construction we can keep the pair $(Y_1,Z_1)$ independent of $(Y_2,Z_2)$. Then equations \ref{eq:couplingXYZ} and \ref{eq:couplingZZZ} can be used to define a joint distribution in $\Omega$.
We then derive the following inequality, 
\begin{align*}
    d_{W_2}\left(u\sharp\Psi, \Phi\right)^2 &=  \underset{(u^\top X,Z) \in \Omega}{inf} \bbE \left[(u^\top X-Z)^2 \right] \\
    &\leq \underset{(Y_1,Z_1) \in \Omega_1 \textit{ and } (Y_2,Z_2) \in \Omega_2}{inf} \bbE \left[(\alpha Y_1 + \sqrt{1-\alpha^2} Y_2 - \alpha Z_1 - \sqrt{1-\alpha^2} Z_2)^2 \right] \\
    &= \underset{(Y_1,Z_1) \in \Omega_1 \textit{ and } (Y_2,Z_2) \in \Omega_2}{inf} \alpha^2\bbE \left[(Y_1 - Z_1 )^2 \right] + (1-\alpha^2)\bbE \left[(Y_2 -  Z_2)^2 \right] + 2\alpha\sqrt{1-\alpha^2}\bbE \left[(Y_1- Z_1)(Y_2 -  Z_2) \right] \\
    &= \underset{(Y_1,Z_1) \in \Omega_1 \textit{ and } (Y_2,Z_2) \in \Omega_2}{inf} \alpha^2\bbE \left[(Y_1 - Z_1 )^2 \right] + (1-\alpha^2)\bbE \left[(Y_2 -  Z_2)^2 \right] \\ 
    &= \alpha^2 \underset{(Y_1,Z_1) \in \Omega_1}{inf} \bbE \left[(Y_1 - Z_1)^2 \right]  +  (1-\alpha^2)\underset{(Y_2,Z_2) \in \Omega_2}{inf}\bbE \left[(Y_2 -  Z_2)^2 \right]\\
    &= \alpha^2 d_{W_2} (v.\Psi, \Phi)^2 + (1-\alpha^2)d_{W_2} (w\sharp\Psi, \Phi)^2\\
    &\leq \alpha^2 d(\Psi, \Phi)^2 + (1-\alpha^2)d(\Psi_{|W}, \Phi)^2.
\end{align*}
Rewriting the above we get
\[\norm{\textit{Proj}_{W^{\perp}}(u)}^2 = \alpha^2 \geq \frac{d_{W_2}\left(u\sharp\Psi, \Phi\right) - d (\Psi_{|W}, \Phi)^2}{d(\Psi, \Phi)^2 - d(\Psi_{|W}, \Phi)^2}. \]

 Using Theorem~\ref{thm:nongaus} and the hypothesis, we have for an appropriate $\gamma_{\sigma,\epsilon}$ and $C_{\sigma}$ that $d_{W_2}(u\sharp\Psi, \Phi) \geq \sqrt{(1-\delta^2)}d(\Psi,\Phi)$
 \[\norm{\textit{Proj}_{W^{\perp}}(u)}^2 \geq \frac{d_{W_2}\left(u\sharp\Psi, \Phi\right) - d (\Psi_{|W}, \Phi)^2}{d(\Psi, \Phi)^2 - d(\Psi_{|W}, \Phi)^2} \geq\frac{(1-\delta^2)d(\Psi,\Phi)^2 - d (\Psi_{|W}, \Phi)^2}{d(\Psi, \Phi)^2 - d(\Psi_{|W}, \Phi)^2}.\]

 Finally, we get 
 $\norm{\textit{Proj}_{W}(u)} = \sqrt{1-\norm{\textit{Proj}_{W^{\perp}}(u)}^2} \leq \delta\frac{d(\Psi,\Phi)}{\sqrt{d(\Psi, \Phi)^2 - d(\Psi_{|W}, \Phi)^2}}.$
\end{proof}

\newpage
\bibliographystyle{alpha}
\bibliography{wassersten-colt}

\newcommand{\etalchar}[1]{$^{#1}$}
\begin{thebibliography}{CCC{\etalchar{+}}03}

\bibitem[BKN18]{bickel2018projection}
Peter~J Bickel, Gil Kur, and Boaz Nadler.
\newblock Projection pursuit in high dimensions.
\newblock {\em Proceedings of the National Academy of Sciences},
  115(37):9151--9156, 2018.

\bibitem[BL19]{bobkov2019one}
Sergey Bobkov and Michel Ledoux.
\newblock {\em One-dimensional empirical measures, order statistics, and
  Kantorovich transport distances}, volume 261.
\newblock American Mathematical Society, 2019.

\bibitem[CCC{\etalchar{+}}03]{cao2003comparison}
LJ~Cao, Kok~Seng Chua, WK~Chong, HP~Lee, and QM~Gu.
\newblock A comparison of pca, kpca and ica for dimensionality reduction in
  support vector machine.
\newblock {\em Neurocomputing}, 55(1-2):321--336, 2003.

\bibitem[DF84]{diaconis1984asymptotics}
Persi Diaconis and David Freedman.
\newblock Asymptotics of graphical projection pursuit.
\newblock {\em The annals of statistics}, pages 793--815, 1984.

\bibitem[FT74]{friedman1974projection}
Jerome~H Friedman and John~W Tukey.
\newblock A projection pursuit algorithm for exploratory data analysis.
\newblock {\em IEEE Transactions on computers}, 100(9):881--890, 1974.

\bibitem[GF96]{girolami1996negentropy}
Mark Girolami and Colin Fyfe.
\newblock Negentropy and kurtosis as projection pursuit indices provide
  generalised ica algorithms.
\newblock In {\em Advances in Neural Information Processing Systems Workshop},
  volume~9. Denver, CO, 1996.

\bibitem[Hub85]{huber1985projection}
Peter~J Huber.
\newblock Projection pursuit.
\newblock {\em The annals of Statistics}, pages 435--475, 1985.

\bibitem[JC16]{jolliffe2016principal}
Ian~T Jolliffe and Jorge Cadima.
\newblock Principal component analysis: a review and recent developments.
\newblock {\em Philosophical transactions of the royal society A: Mathematical,
  Physical and Engineering Sciences}, 374(2065):20150202, 2016.

\bibitem[Jol02]{jolliffe2002principal}
Ian~T Jolliffe.
\newblock {\em Principal component analysis for special types of data}.
\newblock Springer, 2002.

\bibitem[LL98]{lee1998independent}
Te-Won Lee and Te-Won Lee.
\newblock {\em Independent component analysis}.
\newblock Springer, 1998.

\bibitem[MZ22]{montanari2022overparametrized}
Andrea Montanari and Kangjie Zhou.
\newblock Overparametrized linear dimensionality reductions: From projection
  pursuit to two-layer neural networks.
\newblock {\em arXiv preprint arXiv:2206.06526}, 2022.

\bibitem[NA08]{novey2008complex}
Michael Novey and Tulay Adali.
\newblock Complex ica by negentropy maximization.
\newblock {\em IEEE Transactions on Neural Networks}, 19(4):596--609, 2008.

\bibitem[Wai19]{wainwright2019high}
Martin~J Wainwright.
\newblock {\em High-dimensional statistics: A non-asymptotic viewpoint},
  volume~48.
\newblock Cambridge university press, 2019.

\end{thebibliography}

\appendix
\section{Proof of Corollary~\ref{thm:k-col}}\label{app:cor_k-col}
Here we restate Corollary~\ref{thm:k-col} for convenience.
\begin{corollary*}
    Let $U$ be a $k$-dimensional sub-space of $\bbR^p$, where $k$ is a constant. Let $X_1, ..., X_n$ be $n$ data points from $\Psi$. Let $l < k$ be some integer. Let $v_1 , ..., v_l$ be some vectors in $\bbR^p$. Then there exists some constant $C_{\sigma}$, depending on $\sigma$ such that given $\epsilon > 0$ there exists with high probability a unit vector, $u$ which is orthonormal to all $v_i$ such that 
        \[
    d_{W_2}\left(\frac{\Sigma_{i=1}^n \delta_{u^\top X_i}}{n}, \Phi\right) \geq d_{min}(\Psi_{|U},\Phi) - \epsilon -  \frac{C_{\sigma}}{\sqrt[4]{n}}.
    \]
\end{corollary*}
\begin{proof}[Proof of Corollary~\ref{thm:k-col}]
    As $l < k$ and $\dim(W) = k$, by rank nullity theorem, there exists a unit vector $u$ in $U$ which is orthogonal to all the vectors $\{v_1, ..., v_l\}$. As $k$ is constant. $k \ll n$ thus we can invoke Proposition~\ref{thm:pconc-const} with $p=k$ to get that with high probability
    \[
    \sup_{v : v \in U \textit{ and } \norm{v} = 1} \abs{d_{W_2}\left(\frac{\Sigma_{i=1}^n \delta_{v^\top X_i}}{n}, v.\Psi_{\Phi}\right) - \bbE_{X_1, ..., X_n}\left[d_{W_2}\left(\frac{\Sigma_{i=1}^n \delta_{v^\top X_i}}{n}, v.\Psi_{\Phi}\right)\right]}  < \epsilon.
    \]
    In particular then
    \[
    \abs{d_{W_2}\left(\frac{\Sigma_{i=1}^n \delta_{u^\top X_i}}{n}, u\sharp\Psi_{\Phi}\right) - \bbE_{X_1, ..., X_n}\left[d_{W_2}\left(\frac{\Sigma_{i=1}^n \delta_{u^\top X_i}}{n}, u\sharp\Psi_{\Phi}\right)\right]}  < \epsilon.
    \]
    
    Finally similar to the proof of Theorem~\ref{thm:nongaus} invoking Proposition~\ref{lem:sampleconverge} gives us
    \[
    \bbE_{X_1, ..., X_n}\left[d_{W_2}\left(\frac{\Sigma_{i=1}^n \delta_{u^\top X_i}}{n}, u\sharp\Psi_{\Phi}\right)\right] \leq \frac{C_{\sigma}}{\sqrt[4]{n}}.
    \]
    
    Therefore we have
    \begin{align*}
        d_{W_2}\left(\frac{\Sigma_{i=1}^n \delta_{u^\top X_i}}{n}, \Phi\right) &\geq d_{W_2}\left(u\sharp\Psi,\Phi\right) -  d_{W_2}\left(\frac{\Sigma_{i=1}^n \delta_{u^\top X_i}}{n}, u\sharp\Psi\right)  \\
        &\geq d_{min}(\Psi_{|U},\Phi) - \epsilon - \bbE\left[d_{W_2}\left(\frac{\Sigma_{i=1}^n \delta_{u^\top X_i}}{n}, u\sharp\Psi\right)\right] \\
        &\geq d_{min}(\Psi_{|U},\Phi)  - \epsilon - \frac{C_{\sigma}}{\sqrt[4]{n}}
    \end{align*}
 \end{proof}
    
\section{Proof of Corollary~\ref{thm:choosing-k}} \label{app:cor_choosing-k}
We will need the following simple linear algebraic lemma.\begin{lemma} \label{lem:emptyint}
    Let $v_1, ..., v_k$ be a set of orthonormal vectors in a vector space $V$. Let $G$ be the subspace spanned by $v_1, ..., v_k$. Let $H$ be a subspace of $V$. Then $\forall g \in G$ and $h \in H$ such that $\norm{g} = \norm{h} = 1$, we have that \[\left(g^Th\right)^2 \leq \sum_{j=1}^k \norm{\textit{Proj}_{H}(v_j)}^2\] 
\end{lemma}

\begin{proof}
    We first remember from basic linear algebra that for any unit vectors $v \in V$ and $h \in H$, we have $\abs{v.h} \leq \abs{\textit{Proj}_{H}(v)}.$ Then we can write $g = \sum_{j=1}^k \alpha_j v_j$, where $\sum_j \alpha_j^2 = 1$ as $\norm{g} = 1$ and $v_1, \dots, v_k$ form an orthonormal basis of $G$. Combining these we get
    \begin{align*}
       (g^Th)^2 &= \left(\sum_{j=1}^k  \alpha_j v_j^Th\right)^2 \\
        &\leq \left(\sum_{j=1}^k  \alpha_i^2 \right)\left(\sum_{j=1}^k (v_j^Th)^2  \right) &\textit{by Cauchy-Schwarz} \\
        &\leq \sum_{j=1}^k \norm{\textit{Proj}_{H}(v_j)}^2 \\
    \end{align*}
\end{proof}

Armed with this we restate Corollary~\ref{thm:choosing-k} and prove it:
\begin{corollary*}
       Given  integers $m > k+1$, let $\delta$ be a positive real number such that $4\delta^2 < \frac{1}{m} \left(1 - \frac{d(\Psi_{|W}, \Phi)^2}{d(\Psi, \Phi)^2}\right).$ Let $X_1, ..., X_n$ be $n$ data points from $\Psi$. Let $n$,$p$ go to infinity in a way such that $p/n \rightarrow \gamma$ Given $\epsilon > 0$ there is a $\gamma_{\sigma,\epsilon}$, where $\gamma_{\sigma,\epsilon}$ is a constant depending on $\sigma$, $\epsilon$, such that if $\gamma \leq \gamma_{\sigma,\epsilon}$ then with high probability there \textbf{does not exist} a set of $k+1$ orthonormal unit vectors $v_1, ..., v_{k+1}$ such that 
    \[d_{W_2}\left(\frac{\Sigma_{i=1}^n \delta_{v_j^TX_i}}{n}, \Phi\right) \geq \sqrt{1-4\delta^2} d(\Psi,\Phi) + \epsilon + \frac{C_{\sigma}}{\sqrt[4]{n}}. \]    
\end{corollary*}

\begin{proof}[Proof of Corollary~\ref{thm:choosing-k}]
 To prove this we will use Theorem~\ref{thm:altmain} along with the trivial linear algebraic Lemma \ref{lem:emptyint}. We prove by contradiction. Suppose a orthonormal set $v_1, ..., v_{k+1}$ exists satisfying the hypothesis :  
 \[
 d_{W_2}\left(\frac{\Sigma_{i=1}^n \delta_{v_{j}^TX_i}}{n}, \Phi\right) \geq \sqrt{1-4\delta^2} d(\Psi,\Phi) + \epsilon + \frac{C_{\sigma}}{\sqrt[4]{n}}. 
 \] 
 We can invoke Theoremm~\ref{thm:altmain} to get with high probability, \[\norm{\textit{Proj}_{W}(v_j)} \leq 2\delta \frac{d(\Psi,\Phi)}{\sqrt{d(\Psi, \Phi)^2 - d(\Psi_{|W}, \Phi)^2}}.\]
 Then let $G = \textit{Span}\{v_1, ..., v_{k+1}\}$. As As $\textit{dim}(G) +\textit{dim}(W) = \textit{dim}(V)+1$, 
there exists a non-zero vector $s \in G \cap W$ such that $\norm{s} = 1$.
 Invoking Lemma \ref{lem:emptyint} with $H = W$ and $g=h=s$, we get the contradiction 
 \[ 1 = \left(s^Ts\right)^2 \leq \sum_{j=1}^{k+1} \norm{\textit{Proj}_{H}(v_j)}^2 \leq \frac{4(k+1)\delta^2 d(\Psi,\Phi)^2}{d(\Psi, \Phi)^2 - d(\Psi_{|W}, \Phi)^2}  < 1. \]
\end{proof}

\section{Proof of Lemma~\ref{lem:w_2-conc}} \label{app:w_2-conc}
To prove Lemma~\ref{lem:w_2-conc}, we need a bound on the concentration function of subgaussian random variables. For a Borel set $A$, let $A^r$ denote the $r$-fattening of $A$: 
\[
    A_r = \{ x : d(x, A) < r\}.
\]
Let $\mu$ be a probability measure on $\R$. Let
\[
    \alpha_{\mu}(r) = \{1 - \mu(A^r) : \mu(A) \ge 1/2\}, r > 0,
\]
denote the concentration function of $\mu$. 
\begin{lemma}\label{lem:subg-conc-func}
Let $\mu$ be a $\sigma$-subgaussian probability measure. Then there exist absolute constants $C, c > 0$ such that $\alpha_{\mu}(r) \le C e^{-\frac{cr^2}{\sigma^2}}$ for all $r > 0$.
\end{lemma}
\begin{proof}
Without loss of generality, we may assume that $\sigma = 1$. Choose $r_0$ such that $\mu((-r_0, r_0)^c) < \frac{1}{2}$. Then any $A$ such that $\mu(A) > \frac{1}{2}$ must intersect $(-r_0, r_0)$, for otherwise one would get $\mu(A) \le \mu((-r_0, r_0)^c) < \frac{1}{2}$. Take $x_0 \in A \cap (-r_0, r_0)$. Then one must have
\[
    (-r, r) \subseteq x_0 + (-(r_0 + r), r_0 + r) \subseteq A^{r_0 + r}.
\]
Now, by subgaussianity, there exist constants $C_1, c_1 > 0$ such that $\mu((-r, r)^c) \le C_1 e^{-c_1 r^2}$ for all $r > 0$. Therefore
\[
    1 - \mu(A^{r + r_0}) \le \mu((-r, r)^c) \le C_1 e^{-c_1 r^2} \le C e^{-c(r + r_0)^2},
\]
where the last inequality holds for some constants $C, c > 0$ for all large enough $r$, say $r > r_1$. (For example, one can take $c = \frac{1}{2}c_1, C = C_1 e^{\frac{1}{2}B r_0^2}$ and $r_1 = 2r_0$.) Thus for all $r > r_0 + r_1$, we have that $\alpha_{\mu}(r) \le C e^{-cr^2}$.

We can always increase the constant $C$ so that one has $\sup_{r \in (0, r_0 + r_1]}\alpha_{\mu}(r) \le C e^{-c(r_0 + r_1)^2}$. Then for any $r \le r_0 + r_1$,
\[
    \alpha_{\mu}(r) \le \sup_{r \in (0, r_0 + r_1]}\alpha_{\mu}(r) \le Ce^{-c(r_0 + r_1)^2} \le Ce^{-cr^2}.
\]
We conclude that there exist absolute constants $C, c > 0$ such that $\alpha_{\mu}(r) \le Ce^{-cr^2}$ for all $r > 0$.
\end{proof}

For clarity's sake we restate Lemma~\ref{lem:w_2-conc}.
\begin{lemma*}
Let $\mu$ be a $\sigma$-subgaussian  measure. Let $\mu_n$ be the empirical measure formed from an i.i.d. sample of size $n$ from $\mu$. Then \[
    \bbP\big(\abs{d_{W_2}(\mu_n, \mu) - \bbE d_{W_2}(\mu_n, \mu)} \geq t \big) \leq Ce^{-\frac{cnt^2}{\sigma^2}}.
\]
for some absolute constants $C, c > 0$.
\end{lemma*}

\begin{proof}[Proof of Lemma~\ref{lem:w_2-conc}]
The proof is the same as the proof of Theorem 7.1 in \cite{bobkov2019one}, except that we replace their log-Sobolev assumption on $\mu$ with a subgaussianity assumption, which yields a stronger bound on the concentration function as in Lemma~\ref{lem:subg-conc-func}, which in turns gives us a tail bound of the form
\[
    \bbP\big(\abs{d_{W_2}(\mu_n, \mu) - \bbE d_{W_2}(\mu_n, \mu)} \geq t \big) \le Ce^{-\frac{cnt^2}{\sigma^2}}
\]
for some absolute constants $C, c > 0$.
\end{proof}
\end{document}